\documentclass[10pt]{article} 

\usepackage[accepted]{rlj} 

\usepackage{amssymb}            
\usepackage{mathtools}          
\usepackage{mathrsfs}           
\usepackage{graphicx}           
\usepackage{subcaption}         
\usepackage[space]{grffile}     
\usepackage{url}                

\usepackage{amsthm}
\usepackage{thmtools,thm-restate}
\DeclareMathOperator*{\argmax}{arg\,max}

\newtheorem{corollary}{Corollary}

\title{Reinforcement Learning under State and Outcome Uncertainty: A Foundational Distributional Perspective}

\setrunningtitle{Foundations of DistRL under State and Outcome Uncertainty}

\author{Larry Preuett\textsuperscript{1}, Qiuyi Zhang\textsuperscript{2}, Muhammad Aurangzeb Ahmad\textsuperscript{3}}

\emails{lpreuett@uw.edu, qiuyiz@google.com, maahmad@uw.edu}

\affiliations{
$^{1}$\textbf{University of Washington, Tacoma, USA}\\
$^{2}$\textbf{Google DeepMind, California, USA}\\
$^{3}$\textbf{University of Washington, Bothell, USA}\\
}

\contribution{
    We prove that both distributional partially observable evaluation $\widetilde{\mathcal{T}}_{PO}$ and optimality $\widetilde{\mathcal{T}}^*_{PO}$ operators are $\gamma$-contractions under the supremum $p$-Wasserstein metric.
}{
    This extends prior DistRL results for fully observable MDPs \citep{dist_rl_book} to POMDPs, ensuring convergence and a unique fixed point in the risk-neutral setting.
}

\contribution{
    We introduce $\psi$-vectors, distributional analogs of $\alpha$-vectors, and prove that a finite set of these $\psi$-vectors suffices to represent the optimal distributional value function under risk-neutral objectives.
}{
    This parallels the classical result that a finite set of $\alpha$-vectors captures the optimal value function in POMDPs \citep{sondik}. Here, $\psi$-vectors preserve the piecewise linear and convex structure in the Wasserstein metric space, ensuring computational tractability.
}

\contribution{
    We introduce Distributional Point-Based Value Iteration (DPBVI), which replaces the classical POMDP operator with a distributional partially observable Bellman optimality operator and substitutes $\alpha$-vectors with $\psi$-vectors to learn a finite representation of the optimal distributional value function under risk-neutral objectives.
}{
    By leveraging $\psi$-vectors in lieu of $\alpha$-vectors, DPBVI preserves a piecewise linear and convex distribution of returns while retaining the point-based backup structure of PBVI \citep{pbvi}. This extension is the first to unify partial observability and DistRL in a computationally tractable manner.
}

\contribution{
    In the risk-neutral setting, DPBVI’s $\psi$-vectors recover the same policy as PBVI by matching its value function in expectation, and we release open-source code to foster reproducibility and future extensions.
}{
    DPBVI reproduces PBVI’s classical solution by taking the expectation of each $\psi$-vector, thus unifying distributional and classical POMDP methods.
}

\summary{In many real-world planning tasks, agents must tackle uncertainty about the environment’s state and variability in the outcomes of any chosen policy. We address both forms of uncertainty as a first step toward safer algorithms in partially observable settings. Specifically, we extend Distributional Reinforcement Learning (DistRL)—which models the entire return distribution for fully observable domains—to Partially Observable Markov Decision Processes (POMDPs), allowing an agent to learn the distribution of returns for each conditional plan. Concretely, we introduce new distributional Bellman operators for partial observability and prove their convergence under the supremum $p$-Wasserstein metric. We also propose a finite representation of these return distributions via $\psi$-vectors, generalizing the classical $\alpha$-vectors in POMDP solvers. Building on this, we develop Distributional Point-Based Value Iteration (DPBVI), which integrates $\psi$-vectors into a standard point-based backup procedure—\emph{bridging DistRL and POMDP planning}. By tracking return distributions, DPBVI naturally enables risk-sensitive control in domains where rare, high-impact events must be carefully managed. We provide source code to foster further research in robust decision-making under partial observability.}

\begin{document}

\maketitle  

\begin{abstract}
In many real-world planning tasks, agents must tackle uncertainty about the environment’s state and variability in the outcomes of any chosen policy. We address both forms of uncertainty as a first step toward safer algorithms in partially observable settings. Specifically, we extend Distributional Reinforcement Learning (DistRL)—which models the entire return distribution for fully observable domains—to Partially Observable Markov Decision Processes (POMDPs), allowing an agent to learn the distribution of returns for each conditional plan. Concretely, we introduce new distributional Bellman operators for partial observability and prove their convergence under the supremum $p$-Wasserstein metric. We also propose a finite representation of these return distributions via $\psi$-vectors, generalizing the classical $\alpha$-vectors in POMDP solvers. Building on this, we develop Distributional Point-Based Value Iteration (DPBVI), which integrates $\psi$-vectors into a standard point-based backup procedure—\emph{bridging DistRL and POMDP planning}. By tracking return distributions, DPBVI lays the foundation for future risk-sensitive control in domains where rare, high-impact events must be carefully managed. We provide source code to foster further research in robust decision-making under partial observability.
\end{abstract}

\section{Introduction}

Reinforcement Learning (RL) traditionally maximizes expected returns, treating each future reward distribution as a scalar. However, growing evidence suggests that modeling the full return distribution, rather than just its mean, can yield robust policies, better exploration, and richer theoretical insights \citep{dist_rl_book}. This distributional perspective, referred to as Distributional Reinforcement Learning (DistRL), captures variability and higher-order statistics of returns, providing agents with more profound information about environment stochasticity.

While DistRL has been well-studied in fully observable Markov Decision Processes (MDPs), many real-world systems are partially observable \citep{nudgerank,ai_clinician,drone_ai_application}, forcing the agent to infer the latent state from noisy or incomplete observations. Such problems are naturally framed as Partially Observable Markov Decision Processes (POMDPs), but the literature on distributional methods in these settings remains sparse.

Beyond maximizing expected returns, many real-world settings require \emph{risk-sensitive} planning, where limiting the probability of adverse or catastrophic outcomes is crucial. In partially observable domains, the stakes are even higher because the agent’s uncertainty about the state can exacerbate the impact of rare but high-cost events. For instance, a slight misjudgment of the patient’s condition or the vehicle’s surroundings in healthcare or autonomous navigation respectively can lead to undesirable or dangerous consequences. Incorporating risk considerations into POMDP planning explicitly accounts for such worst-case scenarios, making policies more conservative when high losses are possible and balancing exploration against the need to avoid significant harm. Establishing a foundational framework for distributional methods in POMDPs, as we do here, is thus a necessary first step toward risk-sensitive control in partially observable domains. Although we do not formally address risk-sensitive objectives, modeling full return distributions under partial observability is a necessary prerequisite for future work in this direction.

In this work, we present a formal extension of DistRL to POMDPs, bridging the gap between distributional theory and partial observability. Concretely, we make three major contributions:
\begin{itemize}
    \item \textbf{Formal Extension of DistRL to POMDPs.}
    We define the partially observable distributional evaluation operator $\widetilde{\mathcal{T}}_{PO}$ and its optimality operator $\widetilde{\mathcal{T}}^*_{PO}$, showing that both are $\gamma$-contraction under the supremum $p$-Wasserstein metric ($1 \le p < \infty$). This result generalizes classical DistRL convergence results to the partially observable setting.
        \item \textbf{Finite Representation via $\psi$-Vectors.}
    We introduce $\psi$-vectors, a distributional analog to the well-known $\alpha$-vectors in POMDP theory. In the risk-neutral regime, we show that a finite set of $\psi$-vectors suffices to represent the optimal distributional value function while preserving the piecewise linear and convex (PWLC) property in the Wasserstein space.
    \item \textbf{Distributional Point-Based Value Iteration (DPBVI).}
    We adapt Point-Based Value Iteration (PBVI) to the distributional setting, yielding DPBVI, which uses $\psi$-vectors instead of $\alpha$-vectors. Under risk-neutral objectives, DPBVI converges to the same policy as PBVI yet lays the groundwork for risk-sensitive extensions by maintaining full return distributions.
\end{itemize}

Alongside these theoretical results, we release our source code\footnote{Source code available at \href{https://github.com/lpreuettUW/distributional_point_based_value_iteration}{\texttt{github.com/lpreuettUW/distributional\_point\_based\_value\_iteration}}.} to facilitate further exploration of risk-aware approaches in partially observable domains. Thus, our work unifies partial observability and DistRL, enabling distributionally robust decision-making where state uncertainty and return variability are central considerations.

\section{Setting}

We consider a finite, discrete-time Partially Observable Markov Decision Process (POMDP) $\mathcal{P}$. POMDPs are defined as a tuple: $\langle \mathcal{S, A, O}, T, \Omega, b_0, \mathcal{R}, \gamma \rangle$, where $\mathcal{S}$ is a set of discrete states, $\mathcal{A}$ is a set of discrete actions, $\mathcal{O}$ is the discrete space of noisy and/or incomplete state information, $T(s, a, s') = P(s_{t+1} = s' \mid s_t = s, a_t = a)$ is the state transition model, $\Omega(o', s', a) = P(o_{t+1} = o' \mid s_{t+1} = s', a_t = a)$ is the sensor model, $b_0 \in \Delta$ is the initial belief, $\mathcal{R} : \mathcal{S} \times \mathcal{A} \rightarrow \mathbb{R}$ is the reward function, and $0 \le \gamma < 1$ is the discount factor.

In this setting, the agent is unable to directly observe the state of the environment and must rely on a developed belief state $b \in \Delta$, a probability distribution across $\mathcal{S}$. The belief state serves as sufficient information for the agent to behave optimally. The agent's belief is developed using the sequence of observations $b_t = P(s_t \mid b_0, a_0, o_1, \dots, o_{t-1}, a_{t-1}, o_t)$ by

\begin{equation}
    \begin{aligned}
        b_t(s_t) & = \tau(b_{t-1}, a, o) \\
        & = \dfrac{\Omega(o_t, s_t, a_{t-1})\sum_{s_{t-1} \in \mathcal{S}}T(s_{t-1}, a_{t-1}, s_t)b_{t-1}(s_{t-1})}{\sum_{s_t \in \mathcal{S}}\Omega(o_t, s_t, a_{t-1})\sum_{s_{t-1} \in \mathcal{S}}T(s_{t-1}, a_{t-1}, s_t)b_{t-1}(s_{t-1})}
    \end{aligned}
\end{equation}

In POMDPs, the objective is to learn an optimal, stationary policy $\pi^*(a \mid b)$ which maximizes the expected return $\mathbb{E}\left[ \sum_{t=0}^T \gamma^t \mathcal{R}(s_t, a_t) \right]$ for time horizon $T$. Policies in this setting may be viewed as conditional plans and are typically learned by learning the optimal value function 

\begin{equation}
    V^*(b) = \max_{a \in \mathcal{A}}\left[\sum_{s \in \mathcal{S}} \mathcal{R}(s, a)b(s) + \gamma \sum_{o' \in \mathcal{O}} \sum_{s' \in \mathcal{S}} \Omega(o', s', a) \sum_{s \in \mathcal{S}}T(s, a, s')b(s)V^*(b') \right]
\end{equation}

The $n$-th horizon value function is comprised of a set of $\alpha$-vectors $V_n = \{ \alpha_0, \alpha_1, \dots, \alpha_m \}$ and is piecewise linear and convex (PWLC) \citep{sondik}. Each $\alpha$-vector is a $|\mathcal{S}|$-dimensional hyperplane and defines the value function for some bounded region of $\Delta$ (i.e., $\max_{\alpha \in V} \alpha \cdot b$). At each planning step, the next value function $V_n$ may be computed from the previous value function $V_{n-1}$ via the Bellman optimality operator $\mathcal{T}_{PO}$

\begin{equation}
    \begin{aligned}
        V_n(b) & = \mathcal{T}_{PO}V_{n-1} \\
        & = \max_{a \in \mathcal{A}}\left[ \sum_{s \in \mathcal{S}} \mathcal{R}(s, a)b(s) + \gamma \sum_{o' \in \mathcal{O}} \max_{\alpha \in V_{n-1}} \sum_{s' \in \mathcal{S}} \Omega(o', s', a) \sum_{s \in \mathcal{S}}T(s, a, s')b(s)\alpha(s') \right]
    \end{aligned}
\end{equation}

\subsection{Point-Based Value Iteration}

PBVI computes $V_t = \mathcal{T}_{PO}V_{t-1}$ in three steps. First, it creates projections for each action and observation
\begin{equation}
    \begin{aligned}
        & \Gamma^{a,*} \leftarrow \alpha^{a,*}(s) = \mathcal{R}(s, a) \\
        & \Gamma^{a,o'} \leftarrow \alpha_i^{a,o'}(s) = \gamma \sum_{s' \in \mathcal{S}} T(s, a, s') \Omega(o', s', a) \alpha_i^{t-1}(s'), \forall \alpha_i^{t-1} \in V_{t-1}    
    \end{aligned}
\end{equation}
Next, the value of action $a$ at belief $b$ is computed by

\begin{equation}
    \begin{aligned}
        & \Gamma_b^a = \Gamma^{a,*} + \sum_{o' \in \mathcal{O}} \argmax_{\alpha \in \Gamma^{a,o'}}(\alpha \cdot b)
    \end{aligned}
\end{equation}
Lastly, the best action at each belief is used to construct $V_t$

\begin{equation}
    \begin{aligned}
        V_t \leftarrow \argmax_{\Gamma_b^a, \forall a \in \mathcal{A}}(\Gamma_b^a \cdot b), \forall b \in \mathcal{B}
    \end{aligned}
\end{equation}
$V_t$ is comprised of at most $|\mathcal{B}|$ $\alpha$-vectors, thus requiring $|\mathcal{S}||\mathcal{A}||V_{t-1}||\mathcal{O}||\mathcal{B}|$ operations per backup. 

\subsection{Distributional Reinforcement Learning}

Traditionally, the value function models the return expectation. DistRL, instead, aims to learn the compound distribution of the returns $Z^\pi(s)$ sourced to randomness in (1) reward $R$ (2) transition $P^\pi$, and (3) distribution of the next-state value $Z(S')$. Let $\widetilde{\mathcal{T}}$ be the distributional Bellman operator such that

\begin{equation}
    \begin{aligned}
        \widetilde{\mathcal{T}}^\pi Z(s) \overset{D}{=} R + \gamma P^\pi Z(s)
    \end{aligned}
\end{equation}
where $\overset{D}{=}$ denotes equality in distribution, $R$ is the reward random variable, and $P^\pi Z(s) \overset{D}{=} Z(S')$ for $S' \sim T(s, a, \cdot)$.

The distributional Bellman operator has been shown to be a $\gamma$-contraction under the supremum $p$-Wasserstein metric space $\forall p \in [1, \infty]$ \citep{dist_rl_book}. The distributional Bellman optimality operator $\widetilde{\mathcal{T}}^*$, which selects actions that maximize the expected return, similarly converges to the fixed point under the supremum $p$-Wasserstein metric space $\forall p \in [1, \infty]$ \citep{dist_rl_book} if there is a unique optimal policy and a mean-preserving distribution representation $\mathcal{F}$ is used (e.g., a categorical distribution representation). In the case of the existence of multiple optimal policies, however, convergence isn't guaranteed because differing distributions can be similar in expectation.

\section{Distributional Reinforcement Learning Under Partial Observability}

In the presence of partial observability we must account for additional sources of randomness; namely, randomness in (a)\ the unobservable state $\mathcal{S}$, (b)\ observations $O$ and (c)\ sensing $\Omega^\pi$. Therefore, the return distribution is learned with respect to belief instead of state (i.e., $Z(b)$). 

\subsection{The Partially Observable Distributional Bellman Operators}

Let $\widetilde{\mathcal{T}}_{PO}$ be the partially observable distributional Bellman operator such that

\begin{equation}
    \begin{aligned}
        & \widetilde{\mathcal{T}}^\pi_{PO}Z(b) \overset{D}{=} R + \gamma \tau^\pi Z(b)
    \end{aligned}
\end{equation}
where $\tau^\pi Z(b) \overset{D}{=} Z(B')$ for $O \sim \Omega(\cdot, \mid S', A), S' \sim T(\cdot \mid S, A), A \sim \pi(\cdot \mid b)$. The partially observable distributional Bellman optimality operator $\widetilde{\mathcal{T}}^*_{PO}$ may be similarly defined, but where actions are selected to maximize the expected return. The distributional Bellman operators in the partially observable setting share convergence properties with those of the fully observable setting, because any POMDP can be rewritten as a belief MDP, where the state of the belief MDP is the belief of the POMDP.

We denote by $\mathfrak{P}_p(\mathbb{R})$ the set of Borel probability measures on $\mathbb{R}$ with finite $p$-th moment, ensuring that the $p$-Wasserstein distance is well-defined. The space $\mathfrak{P}_p(\mathbb{R})^\Delta$ then denotes the set of functions $\eta\ : \Delta \rightarrow \mathfrak{P}_p(\mathbb{R})$, where $\Delta$ is the belief space of the POMDP.

\begin{restatable}{theorem}{evalOpTheorem}
\label{thrm:pod_operator}
For any POMDP $\mathcal{M}$ with rewards bounded by $[R_{\text{min}}, R_{\text{max}}]$ the partially observable distributional Bellman operator is a contraction on $\mathfrak{P}_p(\mathbb{R})^\Delta$ in the supremum $p$-Wasserstein distance $\Bar{w}_p, \forall p \in [1, \infty), \forall \gamma \in [0, 1)$. That is,
$$ \Bar{w}_p(\widetilde{\mathcal{T}}^\pi_{PO} \eta, \widetilde{\mathcal{T}}^\pi_{PO} \eta') \leq \gamma \Bar{w}_p(\eta, \eta'),$$
$\forall \eta, \eta' \in \mathfrak{P}_p(\mathbb{R})^\Delta$.
\end{restatable}
\noindent\textit{See Appendix~\ref{appendix_sec:pod_operator_proof} for proof.}

Theorem~\ref{thrm:pod_operator} follows from the standard $\gamma$-contraction properties of $\widetilde{\mathcal{T}}^\pi$. Because $\widetilde{\mathcal{T}}^\pi_{PO}$ is a contraction mapping, by the Banach Fixed-Point Theorem, there exists a unique $\eta^* \in \mathfrak{P}_p(\mathbb{R})^\Delta$ such that $\eta^*$ is a fixed point.

\begin{restatable}{theorem}{optOpTheorem}
    \label{thrm:pod_optimality_operator}
    Let \(\widetilde{\mathcal{T}}^*_{PO}\) be the partially observable distributional Bellman optimality operator for a POMDP \(\mathcal{M}\) with bounded rewards and let \(p \in [1,\infty]\). Suppose there is a unique optimal policy \(\pi^*\) and, for each state-action pair $(s, a) \in \mathcal{S} \times \mathcal{A}$, the reward distribution is $P_{\mathfrak{R}}(\cdot \mid s, a)$. For any initial return-distribution function \(\eta_0 \in \mathfrak{P}_p(\mathbb{R})\), the sequence of iterates
    \[
      \eta_{k+1} \;=\; \widetilde{\mathcal{T}}^*_{PO}\,\eta_{k}
    \]
    converges in the supremum \(p\)-Wasserstein distance \(\overline{w}_p\) to \(\eta^{\pi^*}\), the return distribution associated with the unique optimal policy.
\end{restatable}

\begin{corollary}
    \label{coro:pod_opt_convergence}
    Consider a mean-preserving projection $\Pi_\mathcal{F}$ for some representation $\mathcal{F}$. Suppose $\Pi_\mathcal{F}$ is a nonexpansion in the supremum $p$-Wasserstein distance, and let the finite domain $\mathfrak{P}_p(\mathbb{R})^\Delta$ be closed under $\Pi_\mathcal{F}$. If there is a unique optimal policy $\pi^*$ then, under the conditions of Theorem~\ref{thrm:pod_optimality_operator}, the sequence
    $$\eta_{k+1} = \Pi_\mathcal{F}\widetilde{\mathcal{T}}^*_{PO}\eta_k$$
    converges to the fixed point $\eta^{\pi^*}$.
\end{corollary}

These results build on the convergence guarantees of the distributional Bellman optimality operator in \citet[Theorem~7.9]{dist_rl_book}, which assumes a unique optimal policy and analyzes the tabular MDP setting. Our proof adapts these arguments to the belief MDP by ensuring the required continuity and compactness conditions hold, allowing the result to generalize to partially observable settings (see Appendix~\ref{appendix_sec:pod_optimality_operator_proof}). Bellemare et al.\ note that the unique optimal policy assumption is necessary for contraction guarantees but they empirically observe that algorithms like C51 often perform well even when this condition may not strictly hold \cite{dist_rl_book}.

Theorem~\ref{thrm:pod_optimality_operator} states if there is a unique optimal policy, repeatedly applying \(\widetilde{\mathcal{T}}^*_{PO}\) will converge to the optimal return distribution. Corollary~\ref{coro:pod_opt_convergence} establishes that if we approximate the return distribution after each greedy update, the resulting approximate iteration still converges to the same fixed point, provided there is a unique optimal policy.

\subsection{$\psi$-vectors}

In classical POMDP value iteration, each conditional plan---a sequence of actions contingent on future observations---is captured by an $\alpha$-vector, which stores said plan's expected return for each state. Thus, a belief-state's value is found by selecting the $\alpha$-vector with the highest dot product $\alpha \cdot b$. While this approach gives the expected returns of each plan, it does not capture the uncertainty around possible outcomes.

To address this, we adopt a distributional viewpoint: rather than storing a scalar expectation, we store a distribution over returns for each plan. We call these $\psi$-vectors. Formally, a $\psi$-vector is
\begin{equation}
    \Psi = \{\psi^s \mid s \in \mathcal{S}\},
\end{equation}
where $\psi^s$ is a distribution over returns conditioned on being in state s. Given a set of candidate $\psi$-vectors $\Gamma$, the distributional value function at a belief $b$ is then obtained by taking the maximum distributional inner product:
\begin{equation}
    Z(b) \;=\; \max_{a \in \mathcal{A}} \max_{\Psi \in \Gamma_a} \; \big\langle \Psi, b \big\rangle,
\end{equation}
where $\Gamma_a \subset \Gamma$ is the subset of $\psi$-vectors associated with action a. Here, $\langle \Psi, b \rangle$ denotes a mixture distribution of the ${\psi^s}$, weighted by the belief $b$. In the risk-neutral setting, we typically evaluate this mixture by taking its expectation, mirroring the $\alpha \cdot b$ calculation from classical POMDPs---but retaining the entire distribution allows us to capture richer information about potential outcomes.

\begin{restatable}{theorem}{psiPWLC}
    Under the assumptions of Theorem~\ref{thrm:pod_optimality_operator} and Corollary~\ref{coro:pod_opt_convergence}, the optimal distributional value function $Z^*(b)$ can be represented as a finite set of $\psi$-vectors $\Psi$, $\forall a \in \mathcal{A}$, such that for any belief state $b \in \Delta$,
    $$Z^*(b) = \max_{a \in \mathcal{A}}\max_{\Psi \in \Psi^*_a} \langle \Psi, b \rangle,$$
    and these $\psi$-vectors maintain piecewise linearity and convexity within the Wasserstein metric space. 
    \label{thrm:risk_neutral_psi_pwlc}
\end{restatable}
\noindent\textit{See Appendix~\ref{appendix_sec:risk_neutral_psi_pwlc} for proof.}

By Theorem~\ref{thrm:risk_neutral_psi_pwlc}, $Z^*$ inherits the PWLC structure of classical POMDPs, but within the Wasserstein metric space. Crucially, this representation remains finite in the risk-neutral case, ensuring computational tractability---just as a finite set of $\alpha$-vectors suffices to represent the classical optimal value function $V^*$.

\section{Distributional Point-Based Value Iteration}
\label{sec:dpbvi}

Building upon the foundations of PBVI and DistRL in the partially observable setting, we now introduce \emph{Distributional Point-Based Value (DPBVI)}. Conceptually, DPBVI is PBVI with $\widetilde{\mathcal{T}}^*_{PO}$ in place of $\mathcal{T}^*_{PO}$. As in standard PBVI, we retain point-based belief expansion \citep{pbvi}, but our backup step updates a finite set of $\psi$-vectors rather than $\alpha$-vectors.

Recall that in PBVI, we maintain a finite set of belief points $\mathcal{B}\subset \Delta$, where $\Delta$ is the probability simplex over states $\mathcal{S}$. Each backup step creates new $\alpha$-vectors to approximate the PWLC value function, then selects one $\alpha$-vector per belief. DPBVI extends this idea to the distributional setting by replacing $\alpha$-vectors with $\psi$-vectors, where each $\psi$-vector stores one distribution per state for a given action.

By Theorem~\ref{thrm:risk_neutral_psi_pwlc}, a finite set of these $\psi$-vectors suffices to represent the optimal distributional value function $Z^*$ in the risk-neutral setting, just as $\alpha$-vectors suffice in classical POMDPs. We represent return distributions using the categorical distribution parameterization, which is mean-preserving when its support spans all possible returns \citep{dist_rl_book}. This ensures that the expected returns align with the classical POMDP value function, preserving the key convergence guarantees from Corollary~\ref{coro:pod_opt_convergence} while operating in the space of return distributions.

\subsection{The DBPVI Backup}

Given the previous set of $\psi$-vectors $\Gamma^{t-1}$, each DPBVI backup follows four steps: 1.\ action-observation projection, 2.\ categorical distribution projection, 3.\ belief-specific maximization, and 4.\ candidate $\psi$-vector selection.

    We create candidate $\psi$-vectors by combining the immediate rewards $\widetilde{R}_a$ with the next-step distributions of returns, conditioned on actions $a$ and observations $o'$. Formally,
    \begin{equation}
        \Gamma_a^{o',t} \;\leftarrow\; \Bigl\{\widetilde{R}_a + \gamma \,T(\cdot, a, s') \,\Omega(o', s', a)\,\psi^{s'} \,\Big\vert\; \Psi \in \Gamma^{t-1},\,s' \in \mathcal{S}\Bigr\}.
    \end{equation}
    This step yields $|\mathcal{A}| \times |\mathcal{O}| \times |\Gamma^{t-1}|$ projected distributions, denoted by $\Gamma^{,t}$. Because each distribution reflects one action-observation pair, it may be \emph{subnormalized} until we later sum over observations.

    
    The categorical distribution representation is not closed under $\widetilde{\mathcal{T}}^*_{PO}$ due to the $\gamma$-scaling and immediate reward addition operations on prior $\psi$-vectors. This means the subnormalized distributions in $\Gamma^{,t}$ may have heterogeneous supports. To ensure each $\psi$-vector has the same support, we apply the categorical projection operator $\Pi_c$
    $$\Gamma^{,t} \leftarrow \{\Pi_c(\Gamma_a^{o',t}) \mid \Gamma_a^{o',t} \in \Gamma^{,t}\}.$$

    For each belief $b \in \mathcal{B}$, we choose the combination of $\psi$-vectors across observations that yields the highest expected distribution (under the risk-neutral lens). Symbolically, for a single action $a$,
    \begin{equation}
        \Gamma_{a}^{b,t} \;=\; \sum_{o' \in \mathcal{O}} \argmax_{\Psi \in \Gamma_a^{o',t}}\bigl\langle \Psi, b \bigr\rangle.
    \end{equation}
    This step mirrors the cross-sum and maximize operation in PBVI, except each summand is a distributional mixture rather than a scalar $\alpha-$vector. Summing those best subnormalized distributions across all observations recovers a full next-step return distribution.

    Finally, for each belief $b$, we pick the best action's distribution by
    \begin{equation}
        \Psi_b^t
        \;\leftarrow\;
        \argmax_{\Gamma_{a}^{b,t},\, a \in \mathcal{A}}
        \bigl\langle \Gamma_{a}^{b,t}, b \bigr\rangle.
    \end{equation}
    Collecting $\{\psi_b^t\}_{b \in \mathcal{B}}$ yields our updated set $\Gamma^t$. As in PBVI, we keep at most $|\mathcal{B}|$ vectors in $\Gamma^t$, thus maintaining a point-based approximation.

    Because we manage distributions instead of scalars, each backup is slower by a factor of $|M|$, where $|M|$ is the number of bins in each categorical distribution. Concretely, one DPBVI backup runs in $\mathcal{O}(|\mathcal{S}|\;|\mathcal{A}|\;|\Gamma^{t-1}|\;|\mathcal{O}|\;|\mathcal{B}|\;|M|)$, compared to $\mathcal{O}(|\mathcal{S}|\;|\mathcal{A}|\;|V_{t-1}|\;|\mathcal{O}|\;|\mathcal{B}|)$ for PBVI. Nonetheless, DPBVI preserves the core PBVI structure---belief set, backup iteration, and finite policy representation---while enabling distributional modeling of returns under partial observability.

\section{Experimental Results}
\label{sec:experiments}

We empirically validate that DPBVI recovers the same value function as PBVI under risk-neutral objectives, as predicted by theory. Our focus is not on performance or efficiency, but on whether distributional backup operations preserve correctness relative to standard scalar PBVI.

We evaluate DPBVI and PBVI in two small POMDP domains: (1) MiniGrid DoorKey (5×5), where the agent must reach a goal behind a locked door using a key, and (2) a Two-State Noisy-Sensor domain with binary state transitions and stochastic observations. See Appendix~\ref{appendix_sec:experiments} for full environment and implementation details.

Table~\ref{tab:results} reports runtime and iteration counts for both methods. As expected, DPBVI incurs greater computational cost due to distributional operations but converges in the same number of iterations as PBVI. To assess value function agreement, we compare the expected values of DPBVI’s $\psi$-vectors with PBVI’s scalar values across belief points. Figure~\ref{fig:rel_error_1e-3} shows the maximum relative error per iteration. Errors remain very small, confirming that DPBVI tracks PBVI’s value function closely.

We further test for compounding error by tightening the convergence threshold in the Two-State Noisy-Sensor environment (Figure~\ref{fig:rel_error_1e-6}). The relative error initially increases, then stabilizes and decreases, suggesting discrepancies stem from categorical projection, floating-point limits, and bootstrapping rather than theoretical flaws. These findings confirm that DPBVI is a correct extension of PBVI under risk-neutral assumptions and establish a foundation for future risk-sensitive variants.

\section{Conclusion}

In this paper, we presented the first comprehensive study of Distributional Reinforcement Learning in partially observable domains. Our main theoretical result showed that the partially observable distributional Bellman operator $\widetilde{\mathcal{T}}_{PO}^\pi$ and its optimality counterpart $\widetilde{\mathcal{T}}^*_{PO}$ are $\gamma$-contractions in the supremum $p$-Wasserstein metric ($1 \le p < \infty$). These results extend the well-known contraction properties of distributional Bellman operators in fully observable settings to POMDPs. 

As part of our analysis, we introduced $\psi$-vectors as the distributional analogs of $\alpha$-vectors, proving that under risk-neutral control, the optimal distributional value function in a POMDP can be represented by a finite collection of these $\psi$-vectors while preserving PWLC in the Wasserstein metric space. This representation underlines how a distributional perspective naturally generalizes classical POMDP theory.

Building on these foundations, we described DPBVI, a risk-neutral algorithm that applies the distributional Bellman optimality operator to a finite set of belief points. This extension parallels the classical PBVI procedure yet replaces scalar $\alpha$-vectors with distributional $\psi$-vectors, thereby capturing richer information about uncertainty in returns. Experiments (see Section~\ref{sec:experiments} and Appendix~\ref{appendix_sec:experiments}) show that DPBVI recovers the same solution as PBVI under risk-neutral control, confirming that the distributional approach does not compromise solution quality in the risk-neutral regime. 

While our empirical and theoretical results focus solely on the risk-neutral setting, this work enables future exploration of risk-sensitive approaches by extending the distributional perspective to partially observable domains. Capturing full return distributions under uncertainty lays the groundwork for future work on risk-sensitive control. We hope these developments foster further exploration of DistRL in practical domains where both state uncertainty and variability in outcomes are central concerns. 

That said, while our experiments validate theoretical correctness, they are limited to small synthetic environments and do not assess scalability or empirical performance in large-scale settings. Future work may explore practical extensions of DPBVI to high-dimensional problems and risk-sensitive criteria.

\appendix

\section*{Acknowledgments}
\label{sec:ack}
The authors thank Ankur Teredesai and Kevin Jamieson for their helpful feedback on early drafts of this work. Muhammad Aurangzeb Ahmad's work was supported by the National Institutes of Health, USA (NIH grant T32 GM 121290).

\section{Proofs}
\label{appendix_sec:proofs}

\subsection{Proof of Theorem~\ref{thrm:pod_operator}}
\label{appendix_sec:pod_operator_proof}

\evalOpTheorem*

\begin{proof}
    We prove contraction of the partially observable distributional Bellman operator $\widetilde{\mathcal{T}}^\pi_{PO}$ by lifting the POMDP $\mathcal{M}$ to its equivalent fully observable belief MDP $\widetilde{\mathcal{M}}$, where states are beliefs $b \in \Delta$, actions are $a \in \mathcal{A}$, and transitions are governed by the belief update $\tau(b, a, o')$ with probability 
\[
P(o' \mid b, a) = \sum_{s, s'} b(s) T(s, a, s') \Omega(o' \mid s', a).
\]

    Assume the following:
    \begin{itemize}
        \item $\Delta$ is a Borel space (as it is a subset of a finite-dimensional simplex).
        \item $\mathcal{A}$ is finite.
        \item The reward function $R(b, a) := \mathbb{E}_{s \sim b}[R(s, a)]$ is bounded.
        \item The belief transition kernel $\tau(\cdot \mid b, a)$ is measurable.
        \item The policy $\pi: \Delta \to \mathcal{A}$ is fixed and measurable.
    \end{itemize}
    
    Let $\eta, \eta' \in \mathfrak{P}_p(\mathbb{R})^\Delta$ be two distribution-valued value functions.
    Define the supremum $p$-Wasserstein metric as:
    \[
    \bar{w}_p(\eta, \eta') := \sup_{b \in \Delta} W_p(\eta(b), \eta'(b)), \quad \forall p \in [1, \infty).
    \]
    
    Fix $b \in \Delta$ and let $(G_b, G_b')$ be an optimal coupling of $\eta(b)$ and $\eta'(b)$. 
    Define:
    \[
    \tilde{G}_b := R(b, \pi(b)) + \gamma G_{b'}, \quad 
    \tilde{G}_b' := R(b, \pi(b)) + \gamma G'_{b'},
    \]
    where $b' \sim \tau(\cdot \mid b, \pi(b))$ is the next belief under the belief transition kernel.
    
    Then:
    \[
    W_p(\widetilde{\mathcal{T}}^\pi_{PO}\eta(b), \widetilde{\mathcal{T}}^\pi_{PO}\eta'(b))
    \leq \mathbb{E}_{b'} \left[ W_p(\eta(b'), \eta'(b')) \right]
    \leq \gamma \bar{w}_p(\eta, \eta').
    \]
    
    Taking the supremum over $b \in \Delta$, we conclude:
    \[
    \bar{w}_p\left(\widetilde{\mathcal{T}}^\pi_{PO} \eta, \widetilde{\mathcal{T}}^\pi_{PO} \eta'\right)
    \leq \gamma \bar{w}_p(\eta, \eta'),
    \]
    as desired.
\end{proof}
    $$
    $$

\subsection{Proof of Theorem~\ref{thrm:pod_optimality_operator}}
\label{appendix_sec:pod_optimality_operator_proof}

\optOpTheorem*

\begin{proof}
We lift the POMDP \(\mathcal{M}\) to its equivalent belief MDP \(\widetilde{\mathcal{M}}\), where each belief \(b \in \Delta\) serves as a fully observable state. The transition dynamics \(\tau(b, a, o')\) of the belief MDP are Markovian, and the reward function is defined as \(\mathcal{R}(b, a) := \mathbb{E}_{s \sim b}[\mathcal{R}(s, a)]\), which is bounded by assumption. 

From this point forward, we use $b \in \Delta$ to denote the state of the belief MDP. We assume the following conditions hold in the belief MDP:
\begin{itemize}
    \item The belief state space $\Delta$ is compact and convex
    \item The reward function $\mathcal{R}(b, a)$ is continuous in $b\; \forall a \in \mathcal{A}$
    \item The belief update function $\tau(b, a, o')$ is continuous in $b$ for all $a \in \mathcal{A}$ and $o' \in \mathcal{O}$
    \item The action space $\mathcal{A}$ is finite
    \item The policy space $\Pi$ is fixed and measurable
    \item There exists a unique optimal policy $\pi^*$
\end{itemize}

Let \(\eta^{\pi^*}\) denote the return distribution induced by the optimal policy. The partially observable distributional Bellman optimality operator \(\widetilde{\mathcal{T}}^*_{PO}\) can be interpreted as a greedy update rule over return distributions in the belief MDP.

Define $Q_\eta(b, a) := \mathbb{E}_{Z \sim \eta(b, a)} [Z]$ as the expected return under the return distribution $\eta$.
The \emph{action gap} at state $b \in \Delta$ is defined as
$$\text{GAP}(Q_\eta, b) := \min\{ Q_\eta(b, a^*) - Q_\eta(b, a) : a^*, a \in \mathcal{A}, a^* \neq a, Q_\eta(b, a^*) := \max_{a' \in \mathcal{A}} Q_\eta(b, a') \}$$
The global action gap is then defined as
$$\text{GAP} := \inf_{b \in \Delta}\text{GAP}(Q_\eta, b)$$

We now justify that \(Q_\eta(b, a)\) is continuous in \(b\) for all \(a \in \mathcal{A}\). The return distribution $\eta(b, a)$ is defined recursively from the reward distribution and belief transitions. Under our assumptions that the reward function $\mathcal{R}(b, a)$ and belief update $\tau(b, a, o')$ are continuous in b, it follows that \(Q_\eta(b, a) = \mathbb{E}_{Z \sim \eta(b, a)}[Z]\) is continuous in $b$ as well. Since the action space is finite, the maximum and second maximum of $\{Q_\eta(b, a)\}_{a \in \mathcal{A}}$ are continuous in $b$, so $\text{GAP}(Q_\eta, b)$ is also continuous in $b$.

By the Extreme Value Theorem, since $\text{GAP}(Q_\eta, b)$ is continuous on the compact state space $\Delta$, the infimum is attained and equal to the minimum. Since the optimal policy is unique, $\text{GAP}(Q_\eta, b) > 0 \; \forall b \in \Delta$, thus $\text{GAP} > 0$.

Fix $\epsilon = \dfrac{1}{2}\text{GAP}(Q^*)$. The standard Bellman optimality operator is known to be a $\gamma$-contraction under the $L^\infty$ norm. Consequently, $\exists k \in \mathbb{N}$ such that
$$||Q_{\eta_k} - Q^*||_\infty < \epsilon \qquad \forall k \geq K$$

For any fixed $b$, let $a^*$ be the optimal action in that state. Then for any $a \neq a^*$, we have that
$$
\begin{aligned}
    Q_{\eta_k}(b, a^*) &\geq Q^*(b, a^*) - \epsilon\\
    &\geq Q^*(b, a) + \text{GAP}(Q^*) - \epsilon\\
    &> Q_{\eta_k}(b, a) + \text{GAP}(Q^*) - 2\epsilon\\
    &= Q_{\eta_k}(b, a)
\end{aligned}
$$
Thus, any greedy selection rule applied to $\eta_k$ will yield the optimal policy $\pi^* \; \forall k \geq K$. From this point onward, the Bellman updates correspond to evaluation under a fixed policy $\pi^*$, and we may invoke Theorem~\ref{thrm:pod_operator} with the initial return distributions $\eta_0 = \eta_k$ to conclude that $\eta_k \rightarrow \eta^{\pi^*}$. Hence, $\widetilde{\mathcal{T}}^*_{PO}$ is a $\gamma$-contraction and converges to the unique fixed point under the supremum $p$-Wasserstein metric.


\end{proof}


\subsection{Proof of Theorem~\ref{thrm:risk_neutral_psi_pwlc}}
\label{appendix_sec:risk_neutral_psi_pwlc}

\psiPWLC*
    
\begin{proof}
    By corollary~\ref{coro:pod_opt_convergence}, there exists a unique optimal distributional value function $Z^*$. Let $\Gamma^* = \{ \Psi^*_a \mid \forall a \in \mathcal{A} \}$. Then for any belief $b$,
    $$
    \begin{aligned}
        Z^*(b) &= \max_{a \in \mathcal{A}}\max_{\Psi \in \Psi^*_a} \langle \Psi, b \rangle \\
        &\overset{(a)}{=} \max_{\Psi \in \Gamma^*} \langle \Psi, b \rangle \\
        &\overset{(b)}{=} \max_{\alpha \in \Gamma^*} \alpha \cdot b \\
        &\overset{(c)}{=} V^*(b)
    \end{aligned}
    $$
    (a) follows from the definition of $\Gamma^*$ as the union of $\Psi_a$ for all $a$, (b) because of risk-neutrality, $\max_{\Psi \in \Gamma^*} \langle \Psi, b \rangle$ is determined by the maximum expectation. That is, $\argmax_{\Psi} \langle \Psi, b \rangle = \argmax_{\Psi} \mathbb{E}[\langle \Psi, b \rangle]$, where $\alpha = \mathbb{E}[\Psi]$. (c) follows since $\alpha \in \Gamma^*$ precisely matches the expected returns that define the classical optimal value function $V^*$ in a POMDP. Finally, by Sondik's result \citep{sondik}, $V^*$ is PWLC and has finite representation. Therefore, $Z^*$ is PWLC and may be represented by a finite set of $\psi$-vectors.
\end{proof}

\section{Detailed Experimental Results}
\label{appendix_sec:experiments}

Our primary goal in these experiments is to verify that DPBVI converges to the same solution as PBVI under risk-neutral objectives, rather than to improve upon the performance or efficiency of PBVI. Indeed, DPBVI is slower due to the additional cost of handling distributions, but it offers a framework that can be extended to risk-sensitive criteria in future work.

\subsection{Environments}
\label{appendix_sec:envs}

We evaluate on two small POMDP environments, each with its own dynamics models:

\paragraph{MiniGrid DoorKey (5x5)} This environment, adapted from the MiniGrid suite \citep[MiniGrid-DoorKey-5x5-v0]{minigrid}, places an agent in a 5x5 grid containing a locked door and a key. The agent's objective is to reach a goal state behind a locked door requiring a key. We assume both perfect sensor and transition models, a discount factor $\gamma = 0.9$, a reward of $1.0$ upon reaching the goal state, and a known start state. Because the start state is known and the dynamics models are perfect, we set $\mathcal{B}$ to include complete confidence in being in one of the ten states along the optimal path (i.e., $|\mathcal{B}| = 10$).

\paragraph{Two-State Noisy-Sensor} This is a simple domain with two states, $\{ s_0, s_1 \}$, and a noisy sensor adapted from the environment in \citet[Section 17.5]{ai_modern_approach}. The two actions are \emph{Go} and \emph{Stay}. \emph{Go} changes states with probability $0.9$ and \emph{Stay} stays in the same state with probability $0.9$. The agent received a reward of $1.0$ each timestep it is in stay $s_1$ and $0$ otherwise. With probability $0.6$ the sensor reports the correct state. We set the discount factor $\gamma = 0.99$ and $\mathcal{B}$ to include twenty evenly spaced beliefs between complete belief of being in either state (i.e., $|\mathcal{B}| = 20$).

\subsection{Setup}
\label{appendix_sec:experiment_setup}

We disable belief-point expansion in both environments and fix a finite set $\mathcal{B}$ of belief states as described in Section \ref{appendix_sec:envs}. PBVI is implemented as in \citep{pbvi}, while DPBVI follows the backup procedure from Section \ref{sec:dpbvi}. Return distributions in DPBVI use 51 categorical atoms, with each $\psi$-vector initialized to place all probability mass at zero.

For the MiniGrid DoorKey environment, we set the distribution support to $[0, 5]$, reflecting the small range of cumulative rewards in a 5×5 grid. In the Two-State Noisy-Sensor environment, we adopt a larger support $[0, 100]$ to accommodate higher returns. Upon convergence, we compare the expected values of DPBVI’s distributions against the scalar value function from PBVI, using identical discount factors, reward functions, and stopping criteria in both algorithms.

\subsection{Results and Discussion}

Table \ref{tab:results} compares the average runtimes and the number of iterations required for DPBVI and PBVI to converge on the MiniGrid DoorKey and Two-State environments. In both domains, DPBVI recovers the same optimal risk‐neutral solution as PBVI, though with significantly higher runtime due to the additional cost of distributional operations.

\begin{table}[t]
\caption{Results of DPBVI and PBVI: average runtime and iteration counts for convergence criterion $\epsilon = 1\mathrm{e}{-3}$.}
\label{tab:results}
\centering
\begin{tabular}{|c|c|c|c|}
    \hline
    \textbf{Environment} & \textbf{Algorithm} & \textbf{Avg. Runtime (s)} & \textbf{\# Iterations} \\ \hline
    MiniGrid DoorKey & PBVI & 0.014 & 11 \\ \hline
    & DPBVI & 54.805 & 11 \\ \hline
    Two-State & PBVI & 0.083 & 789 \\ \hline
    & DPBVI & 0.397 & 789 \\ \hline
    \end{tabular}
\end{table}


To validate that DPBVI converges to the same value function as PBVI, we examine the maximum relative error between their value functions across belief points over successive backups. These results are shown in Figure~\ref{fig:rel_error_1e-3} for both environments with a convergence criterion of $\epsilon = 1\mathrm{e}{-3}$.

\begin{figure}[t]
    \centering
    \begin{subfigure}[b]{0.45\textwidth}
        \includegraphics[width=\textwidth]{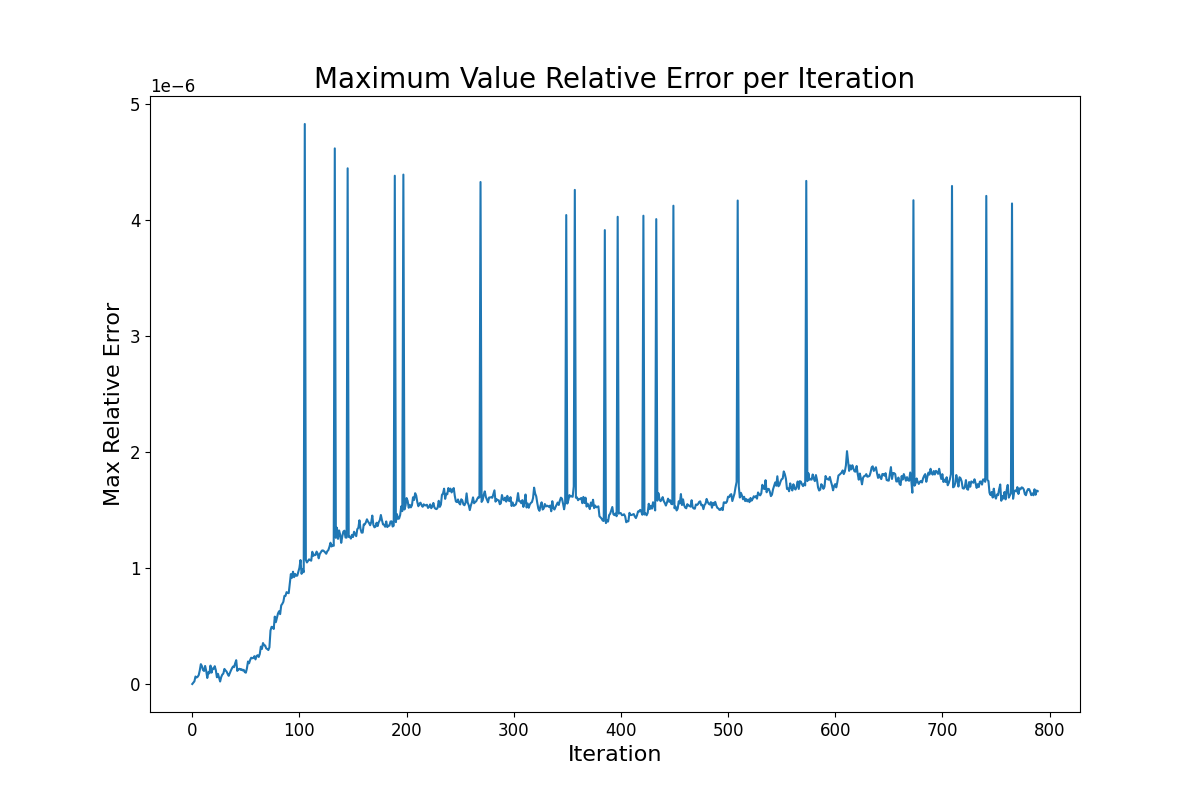}
        \caption{Two-State Noisy-Sensor Environment}
        \label{fig:rel_error_trivial_1e-3} 
    \end{subfigure}
    \begin{subfigure}[b]{0.45\textwidth}
        \includegraphics[width=\textwidth]{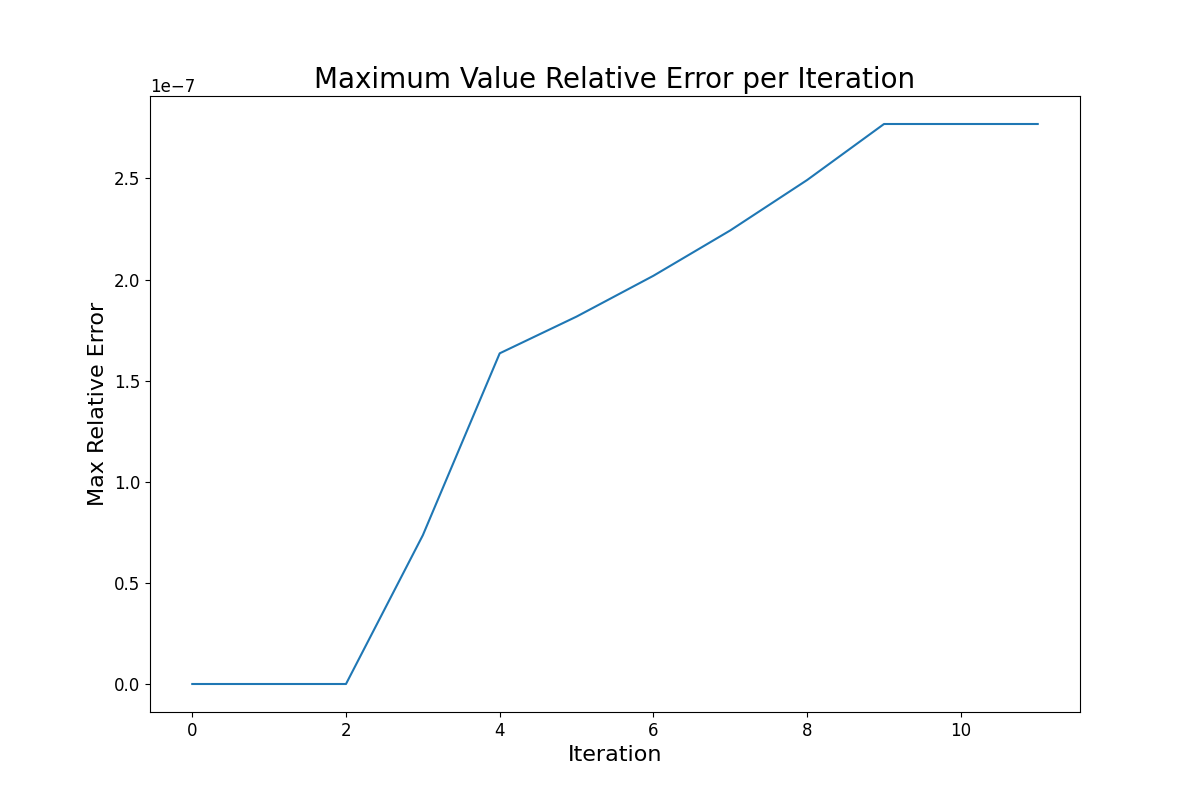}
        \caption{MiniGrid DoorKey Environment}
        \label{fig:rel_error_minigrid_1e-3} 
    \end{subfigure}
    \caption{Maximum relative error between DPBVI and PBVI value functions per iteration with convergence criterion $\epsilon = 1\mathrm{e}{-3}$.}
    \label{fig:rel_error_1e-3}
\end{figure}

In both domains, we observe very small value differences throughout the run, confirming that DPBVI effectively recovers the same solution as PBVI. In the Two-State environment (Figure~\ref{fig:rel_error_trivial_1e-3}), the relative error increases gradually and exhibits intermittent spikes. We hypothesize these spikes arise from categorical projection error, floating-point precision limits, and bootstrapping. In MiniGrid DoorKey (Figure~\ref{fig:rel_error_minigrid_1e-3}), the relative error increases slightly before plateauing around $1.2\mathrm{e}{-7}$, again validating close alignment between both methods.

To test whether the value function discrepancy is a function of the number of backups---suggesting compounding error or a design flaw in DPBVI---we tightened the convergence threshold to $\epsilon = 1\mathrm{e}{-6}$ and set the iteration limit to $10,000$ in the Two-State environment. Results are shown in Figure~\ref{fig:rel_error_1e-6}.

\begin{figure}
    \centering
    \includegraphics[width=0.7\textwidth]{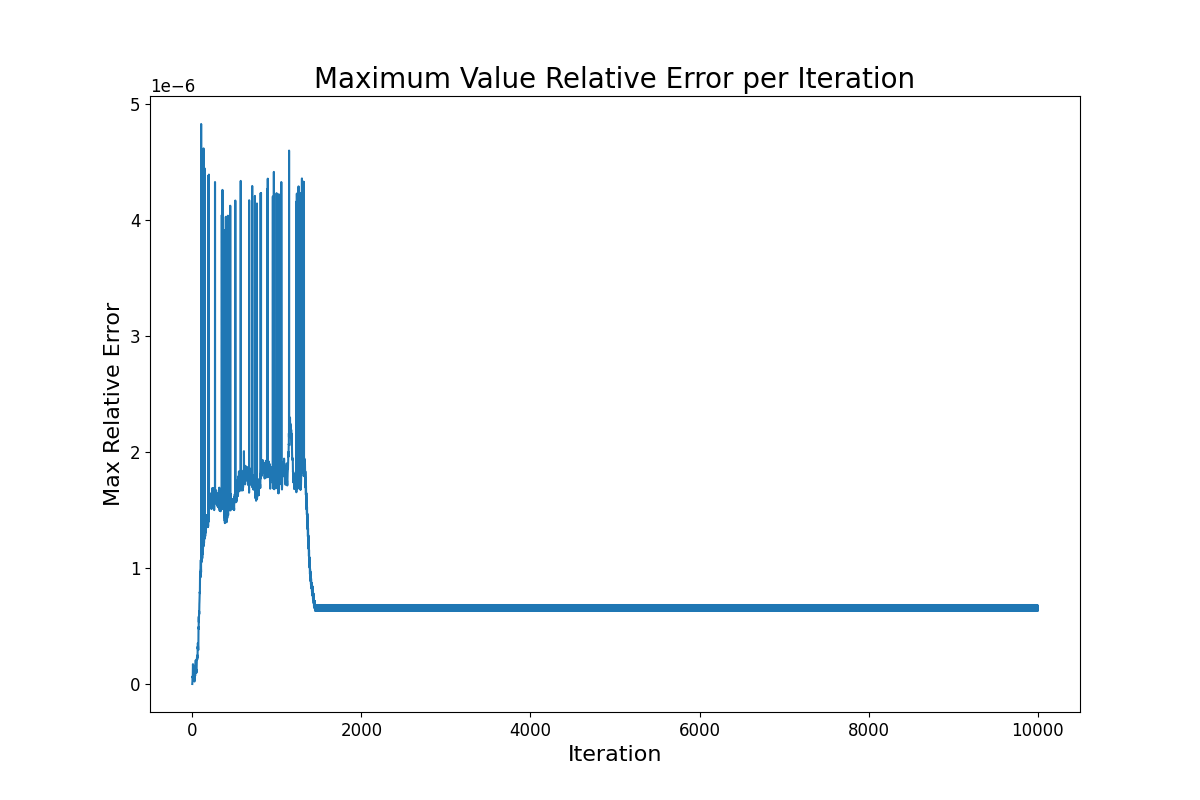}
    \caption{Maximum relative error between DPBVI and PBVI value functions with convergence threshold $\epsilon = 1\mathrm{e}{-6}$ in the Two-State Noisy-Sensor environment.}
    \label{fig:rel_error_1e-6}
\end{figure}

As shown in Figure~\ref{fig:rel_error_1e-6}, the relative error increases over approximately the first $1{,}500$ backups, peaks just below $3\mathrm{e}{-4}$, and then gradually decreases and plateaus near $5\mathrm{e}{-5}$. This behavior supports our hypothesis that the discrepancy arises from bootstrapping, floating-point error, and categorical projection—rather than compounding error over time. Notably, neither PBVI nor DPBVI fully converged under this stricter threshold, yet the error remained bounded and decreased with additional backups. The MiniGrid DoorKey results remained unchanged under this more restrictive convergence criterion.


\bibliography{main}
\bibliographystyle{rlj}


%
%

\end{document}